\documentclass{article}

\usepackage[margin=1in]{geometry}
\usepackage[numbers]{natbib}

\usepackage[utf8]{inputenc} 
\usepackage[T1]{fontenc}    
\usepackage[colorlinks,linkcolor=blue,filecolor=blue,citecolor=black,urlcolor=blue]{hyperref}       
\usepackage{url}    

\usepackage{booktabs}      
\usepackage{amsfonts}       
\usepackage{nicefrac}       
\usepackage{microtype}      
\usepackage[table]{xcolor}         
\usepackage{tocvsec2}
\usepackage{titletoc}

\usepackage{xspace}

\usepackage{amsmath}
\usepackage{amsthm}
\usepackage{amssymb}

\usepackage{multicol}
\usepackage{multirow}
\usepackage{makecell}
\usepackage{enumitem}
\usepackage{wrapfig}
\usepackage{dsfont}

\usepackage{graphicx,epstopdf}
\usepackage{caption}
\usepackage{subcaption} 

\usepackage{algorithm}
\usepackage{algpseudocode}

\usepackage{threeparttable, booktabs}
\usepackage{mathtools}
\usepackage{float}

\usepackage{hyperref}[pdfusetitle]
\hypersetup{colorlinks,urlcolor=blue}
\usepackage{pifont}       
\usepackage{bbding}       

\usepackage{tabularray}
\usepackage{amsmath,amssymb,amsfonts,amsthm,bbm}
\usepackage{algorithm}
\usepackage{algpseudocode}
\usepackage{booktabs}
\usepackage{multirow}

\usepackage{amsmath,amssymb}
\usepackage{booktabs}
\usepackage{array}
\usepackage{adjustbox}
\usepackage{makecell}
\usepackage{hyperref}

\newtheorem{assumption}{Assumption}
\newtheorem{theorem}{Theorem}[section]
\newtheorem{lemma}{Lemma}[section]

\theoremstyle{definition}

\newtheorem{remark}{Remark}[section]
\numberwithin{equation}{section}

\usepackage[skins,listings]{tcolorbox}

\newtcblisting{tikzexample}{
    sidebyside,
    center lower,
    bicolor,
    colbacklower=white,
    sharp corners,
    frame engine=empty
}

\usepackage{quiver}

\restylefloat{algorithm} 
\usepackage{diagbox}

\title{Accelerated Decentralized Stochastic Gradient Descent for Strongly Convex Optimization}

\author{
Ming Sun\\
\texttt{2401111521@stu.pku.edu.cn}\\
Center for Machine Learning Research\\
Peking University
\and
Kun Yuan\thanks{Correspondence to Kun Yuan: \texttt{kunyuan@pku.edu.cn}}\\
\texttt{kunyuan@pku.edu.cn}\\
Center for Machine Learning Research\\
Peking University
}

\begin{document}
\setlength{\parindent}{0pt}
\setlength{\parskip}{0.5em}
\date{}
\maketitle

\begin{abstract}
Decentralized stochastic optimization is a fundamental paradigm for large-scale learning over networks, where agents communicate only with their neighbors and no central coordinator is required. For strongly convex problems, communication efficiency is mainly determined by the condition number \(\kappa=L/\mu\) and the network spectral gap \(1-\beta\). Although deterministic decentralized methods can simultaneously achieve accelerated \(\sqrt{\kappa}\) and \(1/\sqrt{1-\beta}\) dependences, no existing stochastic method attains both improvements at once. In this paper, we propose \emph{Multi-Gossip Accelerated DSGD} (MG-ADSGD), a decentralized stochastic algorithm that combines Nesterov-type primal--dual extrapolation with multi-round fast gossip averaging. The key idea is to couple the gossip depth with the mini-batch size so that additional communication rounds simultaneously improve consensus accuracy and reduce gradient variance. We show that MG-ADSGD achieves the communication complexity
\[
\widetilde{\mathcal O}\!\left(
\frac{\sigma^2}{\mu n\epsilon}\log\frac{1}{\epsilon}
+
\sqrt{\frac{\kappa}{1-\beta}}\log\frac{1}{\epsilon}
\right),
\]
where \(\epsilon\) denotes the target accuracy, \(n\) is the number of nodes, and \(\sigma^2\) is the gradient variance. To the best of our knowledge, this bound yields the best currently available communication complexity for decentralized stochastic strongly convex optimization, up to logarithmic factors that are independent of $\epsilon$.


\end{abstract}

\newpage


\allowdisplaybreaks

\newpage

\section{Introduction}

Decentralized stochastic optimization has become a foundational framework for large-scale machine learning, distributed control, and distributed signal processing. In such systems, each worker holds a local dataset and communicates only with its immediate neighbors according to an underlying network graph, without relying on a central parameter server. This architecture offers several advantages over centralized approaches: it eliminates the single-point communication bottleneck, improves robustness to node failures, and naturally accommodates settings where data are generated and stored at the edge. These properties make decentralized optimization particularly attractive for training large models across geographically distributed data centers, collaborative learning on mobile devices, and multi-agent coordination in networked systems.

A central challenge in decentralized optimization is \emph{communication efficiency}: minimizing the total number of inter-node communication rounds required to reach a target accuracy~$\epsilon$. Communication rounds are often the dominant cost in practice, as they involve synchronization delays and bandwidth consumption that scale with model dimensionality. For strongly convex problems, the communication complexity is governed by two key quantities: the \emph{condition number}~$\kappa = L/\mu$, where~$L$ and~$\mu$ are the smoothness and strong convexity parameters, and the \emph{spectral gap}~$1-\beta$ of the mixing matrix, where~$\beta \in [0,1)$ is its second-largest singular value. A large~$\kappa$ (ill-conditioned objective) or~$\beta$ close to~$1$ (poorly connected network) can each independently inflate communication rounds by orders of magnitude. Designing accelerated algorithms that simultaneously improve the dependence on both quantities is therefore of significant practical and theoretical interest.

In the deterministic setting, this goal has been achieved. The novel work on accelerated decentralized method~\cite{scaman2017optimal} has established that, by incorporating Nesterov-type momentum, one can attain communication complexities with an~$\sqrt{\kappa}$ dependence on the condition number and a~$1/\sqrt{1-\beta}$ dependence on the spectral gap. Compared with non-accelerated baselines, which exhibit linear dependence on both~$\kappa$ and~$1/(1-\beta)$, these results represent a quadratic improvement along each dimension. This body of work provides a clear benchmark: in the absence of stochastic noise, both the optimization and communication dimensions can be simultaneously accelerated.

Whether the same simultaneous acceleration is achievable in the \emph{stochastic} setting, where each worker can only access noisy gradient estimates, remains open. The core difficulty is that gradient noise introduces an error that couples with both the optimization dynamics and the consensus process. Recent works have made partial progress. The method of~\cite{fallah2022robust} accelerates the optimization updates and improves the condition-number dependence, yet its communication complexity is paradoxically \emph{worse} than non-accelerated decentralized SGD, as the acceleration amplifies consensus error. Conversely, reference~\citep{huang2025lmt} incorporates accelerated gossip to improve the spectral-gap dependence to~$1/\sqrt{1-\beta}$, but does not improve the condition-number dependence beyond~$\kappa$. To our best knowledge, \emph{no decentralized stochastic method achieves $\sqrt{\kappa}$ and~$1/\sqrt{1-\beta}$ simultaneously in communication complexity} for strongly-convex optimization problems.

\providecommand{\eps}{\epsilon}

\begin{table}[!ht]
\centering
\caption{Complexity comparison of decentralized methods.}
\label{tab:two-complexities}

\begingroup
\renewcommand{\arraystretch}{1.45}
\setlength{\tabcolsep}{3.5pt}
\scriptsize

\begin{adjustbox}{max width=\textwidth}
\begin{tabular}{>{\raggedright\arraybackslash}p{3.65cm}
                >{\centering\arraybackslash}p{5.85cm}
                >{\centering\arraybackslash}p{5.85cm}}
\toprule
\textbf{Algorithm}
& \textbf{Stochastic gradient computation complexity}
& \textbf{Communication complexity} \\
\midrule

\makecell[l]{
MSDA (deterministic)\\[-1mm]
{\small \cite{scaman2017optimal}}
}
&
$--$
&
$\displaystyle
\mathcal O\!\left(
\sqrt{\frac{L}{\mu(1-\beta)}}\log\frac{1}{\eps}
\right)$
\\

\midrule

\makecell[l]{
Acc-GT (deterministic)\\[-1mm]
{\small \cite{li2024accelerated}}
}
&
$--$
&
$\displaystyle
\mathcal O\!\left(
\sqrt{\frac{L}{\mu(1-\beta)^3}}\log\frac{1}{\eps}
\right)$
\\

\midrule

\makecell[l]{
Acc-GT+CA (deterministic)\\[-1mm]
{\small \cite{li2024accelerated}}
}
&
$--$
&
$\displaystyle
\mathcal O\!\left(
\sqrt{\frac{L}{\mu(1-\beta)}}\log\frac{1}{\eps}
\right)$
\\

\midrule

\makecell[l]{
APM-C (deterministic)\\[-1mm]
{\small \cite{li2020decentralized}}
}
&
$\displaystyle
--$
&
$\displaystyle
\mathcal O\!\left(
\sqrt{\frac{L}{\mu}}\log\frac{1}{\eps}
+
\sqrt{\frac{L}{\mu(1-\beta)}}\log^2\frac{1}{\eps}
\right)$
\\

\midrule



\makecell[l]{
DSGD\\[-1mm]
{\small \cite{lian2017decentralized,yuan2023removing}}
}
&
$\displaystyle
\mathcal O\!\left(
\frac{\sigma^2}{n\mu\eps}
\right)$
&
$\displaystyle
\mathcal O\!\left(
\frac{\sqrt{L}b}{\mu(1-\beta)\sqrt{\eps}}
+
\frac{L}{\mu(1-\beta)}\log\frac{1}{\eps}
\right)$
\\

\midrule

\makecell[l]{
DSGT\\[-1mm]
{\small \cite{pu2021distributed,koloskova2021improved}}
}
&
$\displaystyle
\mathcal O\!\left(
\frac{\sigma^2}{n\mu\eps}
\right)$
&
$\displaystyle
\mathcal O\!\left(
\frac{L}{\mu(1-\beta)}\log\frac{1}{\eps}
\right)$
\\

\midrule

\makecell[l]{
D$^2$ / Exact Diffusion / NIDS\\[-1mm]
{\small \cite{tang2018d2,yuan2023removing}}
}
&
$\displaystyle
\mathcal O\!\left(
\frac{\sigma^2}{n\eps}
\right)$
&
$\displaystyle
\mathcal O\!\left(
\frac{1}{1-\beta}\log\frac{1}{\eps}
\right)$
\\

\midrule

\makecell[l]{
D-MASG\\[-1mm]
{\small \cite{fallah2022robust}}
}
&
$\displaystyle
\mathcal O\!\left(
\frac{\sigma^2}{n\mu^{3/2}\eps}
\right)$
&
$\displaystyle
\widetilde{\mathcal O}\!\left(
\sqrt{\frac{L}{\mu}}\log\frac{1}{\eps}
+
\sqrt{\frac{L}{\mu(1-\beta)}}\frac{1}{\eps^{1/4}}
\right)$
\\

\midrule

\makecell[l]{
LMT (PL condition)\\[-1mm]
{\small \cite{huang2025lmt}}
}
&
$\displaystyle
\mathcal O\!\left(
\frac{L\sigma^2}{n\mu^2\eps}
\right)$
&
$\displaystyle
\mathcal O\!\left(
\frac{L}{\mu\sqrt{1-\beta}}\log\frac{1}{\eps}
\right)$
\\

\midrule

\makecell[l]{
\textbf{MG-ADSGD}\\[-1mm]
{\small Ours}
}
&
$\displaystyle
\mathcal O\!\left(
\frac{\sigma^2}{n\mu\eps}\log\frac{1}{\eps}
\right)$
&
$\displaystyle
\mathcal O\!\left(
\sqrt{\frac{L}{\mu(1-\beta)}}\log\frac{1}{\eps}
\right)$
\\

\bottomrule
\end{tabular}
\end{adjustbox}


\endgroup
\end{table}

In this paper, we close this gap by proposing \emph{Multi-Gossip Accelerated DSGD} (MG-ADSGD), a decentralized stochastic optimization algorithm that combines Nesterov-type primal--dual extrapolation with a multi-round fast gossip averaging (FGA) primitive. The key design idea is to invoke the FGA operator twice per outer iteration with a tunable gossip depth~$R$, which controls the number of internal communication rounds. We further couple the mini-batch size with the gossip depth through the choice~$B = R$, so that increasing~$R$ simultaneously improves communication accuracy and reduces gradient variance. Under standard smoothness, strong convexity, and bounded-variance assumptions, MG-ADSGD achieves a communication complexity of
\[
\widetilde{\mathcal O}\!\left(
\frac{\sigma^2}{\mu n\epsilon} \log \frac{1}{\epsilon}
+
\sqrt{\frac{\kappa}{1-\beta}} \log \frac{1}{\epsilon}
\right),
\]
where $\widetilde{\mathcal O}(\cdot)$ hides all logarithmic factors uncorrelated with the accuracy $\epsilon$. To our knowledge, this bound establishes the best known communication complexity for decentralized stochastic strongly convex optimization, achieving the accelerated~$\sqrt{\kappa}$ dependence on the condition number and the accelerated~$1/\sqrt{1-\beta}$ dependence on the spectral gap simultaneously. The price paid for this acceleration is an additional~$\log(1/\epsilon)$ factor in gradient computation complexity compared to non-accelerated methods, which is typically negligible in communication-dominated regimes. Furthermore, the bound is \emph{tight} in the following sense: when the stochastic noise vanishes ($\sigma^2 = 0$), the communication complexity reduces to~$\widetilde{\mathcal{O}}(\sqrt{\kappa/(1-\beta)}\log(1/\epsilon))$, recovering the rate of the best known deterministic accelerated methods \cite{scaman2017optimal,li2018sharp} up to logarithm factors. Table~\ref{tab:two-complexities} provides a detailed comparison of representative decentralized methods in terms of both gradient computation complexity and communication complexity.

\paragraph{Contributions.}
The main contributions of this work are as follows.
\begin{itemize}
\item We propose MG-ADSGD, which couples Nesterov-type primal--dual extrapolation with multi-round fast gossip averaging, linking the gossip depth~$R$ and mini-batch size~$B=R$ to simultaneously suppress consensus error and gradient variance.

\item We establish that MG-ADSGD achieves the best known communication complexity for decentralized stochastic strongly convex optimization, attaining the accelerated~$\sqrt{\kappa}$ and~$1/\sqrt{1-\beta}$ dependences simultaneously. The bound further recovers the best known deterministic rate when~$\sigma^2=0$.

\paragraph{Notation.}

\end{itemize}

\section{Related Work}
\label{sec:related}

\textbf{Decentralized optimization.} Deterministic decentralized optimization has been studied extensively over the past decades. Foundational first-order approaches include decentralized gradient descent (DGD) \cite{nedic2009distributed}, diffusion  \cite{chen2012diffusion}, and dual averaging  \cite{duchi2011dual}. A common drawback of these methods is the presence of a steady-state bias arising from data heterogeneity across agents \cite{yuan2016convergence}. To overcome this issue, a rich line of work has developed bias-corrected and exactly convergent algorithms, such as ADMM-based primal--dual methods \cite{shi2014linear}, explicit correction methods \cite{shi2015extra, yuan2018exact, li2017decentralized}, and gradient-tracking approaches \cite{xu2015augmented, di2016next, nedic2017achieving, qu2018harnessing}. In the stochastic setting, the seminal work of \cite{lian2017decentralized} showed that decentralized SGD achieves linear speedup with respect to the number of workers. Subsequently, \cite{koloskova2020unified} developed a unified analytical framework for decentralized methods over time-varying communication topologies. More recent works \cite{yuan2020influence,huang2022improving,yuan2023removing,alghunaim2022unified} have demonstrated that data heterogeneity can significantly exacerbate the detrimental effect of sparse network connectivity in decentralized SGD, thereby leading to inferior communication complexity. To mitigate this issue, several correction and variance-reduction mechanisms have been proposed; see, for example, \cite{huang2022improving,yuan2023removing,alghunaim2022unified}. The complexity of some of these methods for strongly convex problems is summarized in Table~\ref{tab:two-complexities}.

\vspace{2mm}
\noindent \textbf{Accelerated decentralized methods.}
A major line of research has focused on achieving optimal communication complexity through algorithmic acceleration. An early milestone was the work of \cite{scaman2017optimal}, which established  tight lower bound \(\Omega(\sqrt{\kappa/(1-\beta)}\log(1/\epsilon))\) for smooth and strongly convex deterministic optimization and proposed the Multi-Step Dual Accelerated method (MSDA) to attain this bound. Subsequent work has sought simpler designs that retain this accelerated rate. In particular, \cite{li2020decentralized} introduced APM, which achieves the optimal rate using only a single consensus step per iteration through an increasing penalty schedule. In the stochastic setting, \cite{fallah2022robust} proposed D-MASG, which integrates decentralized gradient descent with Nesterov momentum to achieve an accelerated transient phase that is robust to noise and heterogeneity, while \cite{huang2025lmt} introduced LMT, which applies momentum to both the local gradient tracker and the global model to simultaneously attain optimal communication and iteration complexities for nonconvex objectives. Despite these advances, the communication complexity of existing stochastic algorithms remains worse than that of their deterministic counterparts; see Table~\ref{tab:two-complexities} for an illustration. This work will design  algorithms that achieve the same communication complexity as their deterministic counterparts in the strongly-convex stochastic setting.

\section{Problem Setup}
\label{sec:problem_setup}

We consider the decentralized stochastic optimization problem
\[
\min_{x\in\mathbb R^d} f(x):=\frac{1}{n}\sum_{i=1}^n f_i(x), \quad \mbox{where} \quad f_i(x)  =\mathbb{E}_{\xi_i \sim D_i}[F_i(x,\xi_i)].
\]
Random variable $\xi_i$ represents data samples at node $i$ that follows local distribution $D_i$. 
Node $i$ has access only to its local objective $f_i$, its local stochastic gradient oracle, and messages received from its graph neighbors.
We use $g_i$ to represent local mini-batch stochastic gradient for each $i$:
\[
g_i(x) = \frac{1}{B} \sum^B_{j=1}\nabla F_i (x,\xi_{ij})\in \mathbb{R}^{d},
\]
and $B$ is the batch-size. Furthermore, we stake all local stochatic gradient into a matrix
\begin{align}\label{eq-g}
G = [g_1^\top;\cdots;g_n^\top]\in\mathbb R^{n\times d}.
\end{align}
Similarly, we have the stacked form of deterministic gradient matrix as
\[
\nabla f(X) = [\nabla f_1 (x_1)^\top ;\cdots;\nabla f_n (x_n)^\top]\in\mathbb R^{n\times d}.
\]
We denote stacked matrix variables at step $k$ as 
\[X^{(k)}=[(x_1^{(k)})^\top;\cdots ;(x_n^{(k)})^\top]\in\mathbb R^{n\times d};\quad Y^{(k)}=[(y_1^{(k)})^\top;\cdots ;(y_n^{(k)})^\top]\in\mathbb R^{n\times d};\]
\[Z^{(k)}=[(z_1^{(k)})^\top;\cdots ;(z_n^{(k)})^\top]\in\mathbb R^{n\times d};\quad G^{(k)}=[(g_1^{(k)})^\top;\cdots ;(g_n^{(k)})^\top]\in\mathbb R^{n\times d}.\]
For variable $X$, we define the network average and disagreement operators by
\[
\bar x:=\frac{1}{n}\sum_{i=1}^n x_i \in \mathbb{R}^d,
\qquad
\Pi:=I-\frac{1}{n}\mathbf 1\mathbf 1^\top \in \mathbb{R}^{n\times n}.
\]
The quantity \(\Pi X\) measures the consensus error. Throughout the paper, we assume each \(f_i(x)\) is strongly convex and denote by \(x^\star\) the unique minimizer of \(f\), with \(f^\star := f(x^\star)\). For each local objective, we define \(f_i^\star := \min_x f_i(x)\). We now introduce the following standard assumptions. 

\begin{assumption}[$L$-smoothness]
\label{ass:smooth}
Each local objective $f_i$ is differentiable and $L$-smooth.
\end{assumption}

\begin{assumption}[$\mu$-strong convexity]
\label{ass:strongly-convex}
Each local objective $f_i$ is $\mu$-strongly convex.
\end{assumption}

\begin{assumption}[Unbiased stochastic gradients with bounded variance]
\label{ass:stochastic-gradients}
The stochastic gradient oracle at node $i$ is unbiased and has bounded conditional variance:
\[
\mathbb E_{\xi_i\sim D_i}[\nabla F_i(x;\xi_i) ]=\nabla f_i(x),
\qquad
\mathbb E_{\xi_i\sim D_i}\bigl[\|\nabla F_i(x;\xi_i)-\nabla f_i(x)\|^2 \bigr]\le \sigma^2.
\]
\end{assumption}

\begin{assumption}[Communication matrix]
\label{ass:mixing}
The communication matrix $W$ is symmetric and doubly stochastic. We also assume $\|W - \frac{1}{n}\mathbf 1\mathbf 1^\top\| \le \beta \in(0,1]$. 
\end{assumption}


\section{MG-ADSGD Algorithm}
\label{sec:algorithm}

This section proposes \emph{Multi-Gossip Accelerated DSGD} (MG-ADSGD). We first state the accelerated gossip primitive used to suppress network disagreement, and then present the full iterative scheme together with its per-iteration oracle and communication costs.

\subsection{Fast Gossip Average Protocol}

\begin{algorithm}[t]
\caption{FastGossipAverage (FGA) : $\mathrm{FGA}_R(V)$}
\label{alg:fga}
\begin{algorithmic}[1]
\Require Local vector $v_i$, stacked matrix $V = [v_1^\top;\cdots;v_n^\top]\in \mathbb{R}^{n\times d}$, mixing matrix $W$, communication rounds $R$, acceleration parameter $\eta$, let $a_i^{(-1)}=a_i^{(0)}=v_i$ for each $i$; 
\For{$r=0,1,\dots,R-1$, every node $i$ }, 
    \State $a_i^{(r+1)} \gets (1+\eta)\sum_{j=1}^n W_{ij}a_j^{(r)}-\eta a_i^{(r-1)}$
\EndFor

\State \Return $[(a_1^{(R)})^\top;\cdots;(a_n^{(R)})^\top]\in \mathbb{R}^{n\times d}$
\end{algorithmic}
\end{algorithm}

Fast Gossip Average (FGA), first proposed in \cite{liu2011accelerated} and subsequently widely used in accelerated algorithms such as \cite{yuan2022revisiting,scaman2017optimal}, is a communication subroutine designed to accelerate consensus over a network; its recursion is given in Algorithm~\ref{alg:fga}. Starting from local vectors maintained by individual nodes, FGA performs multiple rounds of weighted neighbor averaging using an accelerated Chebyshev-type recurrence, rather than standard repeated gossip steps. This construction significantly reduces disagreement among nodes after a fixed number of communication rounds, enabling the network to approximate the global average much faster than plain gossip. As a result, FGA serves as an efficient primitive for decentralized optimization, where accurate averaging is essential but communication is expensive. Its main role is to improve the dependence of the convergence rate on network connectivity by effectively compressing many consensus steps into a small number of communication rounds.

\subsection{Multi-Gossip Accelerated DSGD (MG-ADSGD).}

\begin{algorithm}[t]
\caption{Multi-Gossip Accelerated DSGD (MG-ADSGD)}
\label{alg:mcadsgd}
\begin{algorithmic}[1]
\Require Stepsize $\gamma$, momentum  $\theta$, gossip round $R$, strongly-convex constant $\mu$
\State Initialize $x_i^{(0)}=y_i^{(0)}=z_i^{(0)}=0$ for all $i=1,\dots,n$, stacked matrix $X^{(0)}=Y^{(0)}=Z^{(0)}=[0]^{n\times d}$;
  
\For{$k=0,1,2,\dots,K $}
    \State Sample $\xi_{i,k+1,1},\dots,\xi_{i,k+1,R}$ for each $i$
    \State $g_i^{(k+1)} \gets \frac{1}{R}\sum_{s=1}^{R} \nabla F_i\!\left(y_i^{(k)};\xi_{i,k+1,s}\right)$ for each $i$, and update $G^{(k+1)}$ following \eqref{eq-g}
    \State $Z^{(k+1)} \gets \mathrm{FGA}_R\!\left(\frac{\frac{\theta}{\gamma}Z^{(k)}+\mu Y^{(k)}-G^{(k+1)}}{\frac{\theta}{\gamma}+\mu}\right)$
    \State $X^{(k+1)} \gets \mathrm{FGA}_R\!\left((1-\theta)X^{(k)}+\theta Z^{(k+1)}\right)$
    \State $Y^{(k+1)} \gets (1-\theta)X^{(k+1)}+\theta Z^{(k+1)}$

\EndFor
\end{algorithmic}
\end{algorithm}

MG-ADSGD is designed by integrating three complementary mechanisms:
acceleration in the optimization dynamics, acceleration in the communication
process, and variance control in stochastic gradients; see recursions in Algorithm~\ref{alg:mcadsgd}. 
Specifically, the method adopts a three-sequence accelerated structure
\((X^{(k)}, Y^{(k)}, Z^{(k)})\), in which \(Y^{(k)}\) serves as the
extrapolated query point for gradient evaluation, \(Z^{(k)}\) acts as an
auxiliary descent variable, and \(X^{(k)}\) represents the primal iterate.
This construction is reminiscent of Nesterov-type accelerated schemes for
strongly convex optimization, where gradient information is evaluated at an
extrapolated point to improve the convergence rate. The \(Z\)-update combines
the previous momentum state, the current extrapolated iterate, and the
stochastic gradient through coefficients determined by the stepsize \(\gamma\),
momentum parameter \(\theta\), and strong convexity parameter \(\mu\), thereby
encoding the curvature of the objective into the update rule.

A distinctive feature of the method is that consensus is not enforced by a
single mixing step, but rather by applying the accelerated communication
primitive \(\mathrm{FGA}_R(\cdot)\) to both the \(Z\)- and \(X\)-updates.
This design reflects the principle that, in decentralized optimization,
optimization error and consensus error must be controlled simultaneously. By
inserting accelerated gossip into both stages, the algorithm ensures that
descent information and primal iterates are both sufficiently aligned across
the network before being propagated to the next iteration. Moreover, the use of
the same parameter \(R\) for both the mini-batch size and the number of gossip
rounds provides a deliberate coupling between statistical and communication
accuracy: increasing \(R\) reduces gradient variance through averaging while
also improving consensus through deeper accelerated mixing. Consequently,
MG-ADSGD is best understood as a carefully coordinated scheme in which
optimization acceleration, consensus acceleration, and stochastic variance
reduction are jointly balanced to achieve improved convergence. Since each outer iteration requires \(R\) communication rounds, the total communication budget after \(K\) outer iterations is
\[
T = KR,
\]
which serves as the complexity measure in our subsequent analysis.



\section{Convergence Analysis}
\label{sec:theory}

This section establishes the convergence theorem of the proposed MG-ADSGD method. The theorem below summarizes the dependence on optimization conditioning $(L,\mu)$, network connectivity $(1-\beta)$, stochastic noise level $\sigma^2$, and the target accuracy $\epsilon$. 

\begin{theorem}[Convergence complexity for the strongly convex case]
\label{thm:paper_main_complexity}
Given $L\ge 0$, $n\ge 1$, $\beta\in[0,1)$, and $\sigma>0$, assume that the local objectives $\{f_i\}_{i=1}^n$ satisfy Assumptions~\ref{ass:smooth} and~\ref{ass:strongly-convex} with $\mu>0$, that the stochastic gradient oracles satisfy Assumption~\ref{ass:stochastic-gradients}, and that the gossip matrix satisfies Assumption~\ref{ass:mixing}. Let
$
\Delta_x:=\|x^{(0)}-x^\star\|^2.
$
Fix a target accuracy $\epsilon\in(0,L\Delta_x)$. Run \emph{Multi-Gossip Accelerated DSGD} in Algorithm~\ref{alg:mcadsgd} with parameter choices $(R,\gamma)$ specified in ~\ref{eq:gamma-choice}\ref{eq:R-choice} and $T=KR$. Then the total number of gossip rounds
$
T_\star(\epsilon)
$
needed to attain
$
\mathbb E\bigl[f(\bar x^{(K)})-f(x^\star)\bigr]\le \epsilon
$
satisfies
\begin{equation}
\label{eq:paper_T_complexity}
T_\star(\epsilon)
=\widetilde{\mathcal O}\!\left(
\frac{\sigma^2}{\mu n\epsilon}\ln\!\left(\frac{L\Delta_x}{\epsilon}\right)
+\sqrt{\frac{L/\mu}{1-\beta}}\,\ln\!\left(\frac{L\Delta_x}{\epsilon}\right)
\right).
\end{equation}
Here the $\widetilde{\mathcal O}(\cdot)$ notation suppresses logarithmic factors that are independent of $\epsilon$.
\end{theorem}

\begin{remark}[Recovery of the deterministic accelerated rate]
When the stochastic noise vanishes, i.e., $\sigma^2 = 0$, the bound in~\eqref{eq:paper_T_complexity} simplifies to
\[
T_\star(\epsilon)
=
\widetilde{\mathcal O}\!\left(
\sqrt{\frac{L/\mu}{1-\beta}}\,
\log\!\left(\frac{L\Delta_x}{\epsilon}\right)
\right).
\]
Hence, in the noise-free regime, the communication complexity of MG-ADSGD reduces to the accelerated deterministic rate. In particular, it recovers, up to logarithmic factors hidden in the $\widetilde{\mathcal O}(\cdot)$ notation, the optimal complexity established by~\cite{scaman2017optimal}.
\end{remark}

\begin{remark}[Optimality of the communication complexity]
Up to logarithmic factors, the bound in~\eqref{eq:paper_T_complexity} matches the canonical stochastic term $\sigma^2/(\mu n\epsilon)$ and the accelerated network term $\sqrt{(L/\mu)/(1-\beta)}$. To the best of our knowledge, this gives the best currently available communication complexity for decentralized stochastic strongly convex optimization. In particular, compared with the accelerated stochastic methods discussed in the introduction~\cite{fallah2022robust,yuan2022revisiting,huang2025lmt}, our result is stronger in that it simultaneously achieves the accelerated $\sqrt{\kappa}$ dependence on the condition number and the accelerated $1/\sqrt{1-\beta}$ dependence on network connectivity.
\end{remark}

\begin{remark}[Cost in gradient computation]
The improved communication complexity comes at the cost of only a modest increase in stochastic-gradient complexity, namely an additional factor of order \(\log(1/\epsilon)\) over the optimal centralized stochastic rate. This overhead is mild and typically acceptable in decentralized settings, where communication is the primary bottleneck.
\end{remark}


\section{Conclusion}
\label{sec:conclusion}

In this paper, we proposed MG-ADSGD for decentralized stochastic strongly convex optimization. We demonstrated that a coordinated combination of acceleration in the optimization dynamics, accelerated gossip-based communication, and multi-round stochastic gradient updates can lead to substantially improved communication efficiency. Our analysis establishes a communication complexity bound which, to the best of our knowledge, is the best currently available for decentralized stochastic strongly convex optimization. In particular, the result simultaneously achieves an accelerated \(\sqrt{L/\mu}\) dependence on the condition number and an accelerated \(1/\sqrt{1-\beta}\) dependence on the network spectral gap. Furthermore, in the noise-free regime, the bound recovers the deterministic accelerated communication rate up to logarithmic factors. These results indicate that the proposed coordinated design effectively bridges accelerated deterministic decentralized optimization and stochastic decentralized optimization, while preserving near-optimal dependence on both optimization and network parameters.

\appendix
\section{Appendix}
Throughout this appendix, we use uppercase letters to denote stacked node matrices. Specifically,
\[
\Pi
=
I_n-\frac{1}{n}\mathbf{1}_n\mathbf{1}_n^{\top}
\in\mathbb{R}^{n\times n},
\qquad
X^{(k)},Y^{(k)},Z^{(k)},G^{(k)}
\in\mathbb{R}^{n\times d}.
\]
Thus, \(\Pi X^{(k)}\), \(\Pi Y^{(k)}\), and the analogous quantities denote
the corresponding consensus-error matrices.
The symbols \(x_i^{(k)}\),\(y_i^{(k)}\),\(z_i^{(k)}\) and
\(g_i^{(k)}\) denote the \(i\)-th rows, or equivalently the \(i\)-th local
vectors, of \(X^{(k)}\),\(Y^{(k)}\),\(Z^{(k)}\) and \(G^{(k)}\), respectively. Averaged quantities such
as \(\bar{x}^{(k)}\), \(\bar{y}^{(k)}\), and \(\bar{z}^{(k)}\) are kept in
lowercase because they are vectors in \(\mathbb{R}^d\).

\begin{lemma}\label{lem:fund-smooth}
For any stacked matrix \(Y\in\mathbb{R}^{n\times d}\), let
\[
f(\bar{y};Y)
\triangleq
\frac{1}{n}\sum_{i=1}^n
\left(
f_i(y_i)+
\left\langle \nabla f_i(y_i),\bar{y}-y_i\right\rangle
\right).
\]
Then, for any \(x\in\mathbb{R}^d\),
\begin{align*}
f(x)
&\geq
f(\bar{y};Y)
+
\frac{1}{n}\sum_{i=1}^n
\left\langle \nabla f_i(y_i),x-\bar{y}\right\rangle
+
\frac{\mu}{2}\|x-\bar{y}\|^2
+
\frac{\mu}{2n}\|\Pi Y\|_F^2,
\\
f(x)
&\leq
f(\bar{y};Y)
+
\frac{1}{n}\sum_{i=1}^n
\left\langle \nabla f_i(y_i),x-\bar{y}\right\rangle
+
\frac{L}{2}\|x-\bar{y}\|^2
+
\frac{L}{2n}\|\Pi Y\|_F^2 .
\end{align*}
\end{lemma}

Proof of Lemma~\ref{lem:fund-smooth} can be found in
\cite{jakovetic2014fast}.

\begin{lemma}[\sc Descent lemma]
\label{lem:descent-c-sc}
Suppose Assumptions~\ref{ass:smooth}, \ref{ass:strongly-convex}, and
\ref{ass:stochastic-gradients} hold, and Algorithm~\ref{alg:mcadsgd} is used.
Set \(\theta=\sqrt{\mu\gamma}/2\). If \(\gamma\leq 1/\mu\), define
\[
D_f(x;Y)
\triangleq
f(x)
-
\frac{1}{n}
\sum_{i=1}^n
\left(
f_i(y_i)
+
\langle \nabla f_i(y_i),x-y_i\rangle
\right),
\]
then for any
\(k\geq 0\),
\begin{align}
&
\mathbb{E}\left[f(\bar{x}^{(k+1)})\right]-f(x^\star)
+
\left(
\frac{\theta^2}{2\gamma}
+
\frac{\mu\theta}{2}
\right)
\mathbb{E}\left[
\left\|
\mathbb{E}_{k}\left[\bar{z}^{(k+1)}\right]
-
x^\star
\right\|^2
\right]
\nonumber\\
&\leq
(1-\theta)
\left(
\mathbb{E}\left[f(\bar{x}^{(k)})\right]-f(x^\star)
+
\left(
\frac{\theta^2}{2\gamma}
+
\frac{\mu\theta}{2}
\right)
\mathbb{E}\left[
\left\|
\mathbb{E}_{k-1}\left[\bar{z}^{(k)}\right]
-
x^\star
\right\|^2
\right]
\right)
\nonumber\\
&\quad
+
\frac{
\theta^2\gamma(1+L\gamma)\sigma^2
}{
2(\theta+\mu\gamma)^2 nR
}
\nonumber\\
&\quad
-
\left(
\frac{\theta^2}{2\gamma}
-
\frac{L\theta^2}{2}
\right)
\mathbb{E}\left[
\left\|
\mathbb{E}_{k}\left[\bar{z}^{(k+1)}\right]
-
\bar{z}^{(k)}
\right\|^2
\right]
\nonumber\\
&\quad
-
(1-\theta)
\mathbb{E}\left[
D_f\left(\bar{x}^{(k)};Y^{(k)}\right)
\right]
+
\frac{L}{2n(1-\theta)^{k+1}}
\mathbb{E}\left[
\left\|
\Pi Y^{(k)}
\right\|_F^2
\right].
\label{eq:descent-c-sc}
\end{align}
Here \(\mathbb{E}_k[\cdot]\) denotes conditional expectation with respect to
the \(k\)-th round oracle randomness, conditioning on all previous randomness,
and
\[
\mathbb{E}_{-1}\left[\bar{z}^{(0)}\right]
\triangleq
\bar{z}^{(0)},
\qquad
\mathbb{E}\left[
\left\|
\mathbb{E}_{-1}\left[\bar{z}^{(0)}\right]
-
x^\star
\right\|^2
\right]
\triangleq
\left\|
x^{(0)}-x^\star
\right\|^2 .
\]
\end{lemma}

\begin{proof}[Proof of Lemma~\ref{lem:descent-c-sc}]

Using Lemma~\ref{lem:fund-smooth} with
\(x=\bar{x}^{(k+1)}\) and \(Y=Y^{(k)}\), we have
\begin{align}
f(\bar{x}^{(k+1)})
&\leq
f(\bar{y}^{(k)};Y^{(k)})
+
\frac{1}{n}\sum_{i=1}^n
\left\langle
\nabla f_i(y_i^{(k)}),
\bar{x}^{(k+1)}-\bar{y}^{(k)}
\right\rangle
\nonumber\\
&\quad
+
\frac{L}{2}\|\bar{x}^{(k+1)}-\bar{y}^{(k)}\|^2
+
\frac{L}{2n}\|\Pi Y^{(k)}\|_F^2 .
\label{eqn:bmirew}
\end{align}
By the update formula in algorithm~\ref{alg:mcadsgd},
\[
\bar{x}^{(k+1)}-\bar{y}^{(k)}
=
\theta\left(\bar{z}^{(k+1)}-\bar{z}^{(k)}\right).
\]
Substituting this relation into \eqref{eqn:bmirew} gives
\begin{align}
f(\bar{x}^{(k+1)})
&\leq
f(\bar{y}^{(k)};Y^{(k)})
+
\theta
\left\langle
\frac{1}{n}\sum_{i=1}^n\nabla f_i(y_i^{(k)}),
\bar{z}^{(k+1)}-\bar{z}^{(k)}
\right\rangle
\nonumber\\
&\quad
+
\frac{L\theta^2}{2}
\|\bar{z}^{(k+1)}-\bar{z}^{(k)}\|^2
+
\frac{L}{2n}\|\Pi Y^{(k)}\|_F^2
\nonumber\\
&=
f(\bar{y}^{(k)};Y^{(k)})
+
\theta
\left\langle
\frac{1}{n}\sum_{i=1}^n\nabla f_i(y_i^{(k)}),
x^\star-\bar{z}^{(k)}
\right\rangle
\nonumber\\
&\quad
+
\theta
\left\langle
\frac{1}{n}\sum_{i=1}^n\nabla f_i(y_i^{(k)}),
\bar{z}^{(k+1)}-x^\star
\right\rangle
\nonumber\\
&\quad
+
\frac{L\theta^2}{2}
\|\bar{z}^{(k+1)}-\bar{z}^{(k)}\|^2
+
\frac{L}{2n}\|\Pi Y^{(k)}\|_F^2 .
\label{eqn:vidsgvsd}
\end{align}

Let
\[
\mathcal E^k
\triangleq
\{\zeta_i^{(k,r)}:0\leq r<R,\ i\in[n]\}
\]
be the randomness of the stochastic oracles at the \(k\)-th iteration, with
\(\mathcal E^{-1}=\emptyset\). For any
\(k\geq 0\),
\begin{align}
\label{eqn:noise-bound}
\mathrm{Var}_k(\bar{g}^{(k)})
\leq
\frac{\sigma^2}{nR},
\qquad
\mathrm{Var}_k(\bar{z}^{(k+1)})
\leq
\frac{\gamma^2\sigma^2}{(\theta+\mu\gamma)^2nR}.
\end{align}
Here \(\mathrm{Var}_k(\cdot)\) denotes the conditional variance with respect to
the randomness generated at the \(k\)-th iteration, while all previous randomness
is conditioned on. By the algorithmic design,
\(\{(Z^{(\ell)},X^{(\ell)},Y^{(\ell)})\}_{\ell=0}^k\) is independent of
\(\mathcal E^k\).

For notational simplicity, we write \(\mathbb{E}_k[\cdot]\) and
\(\mathrm{Var}_k(\cdot)\) for the conditional expectation and variance with
respect to the \(k\)-th round oracle randomness \(\mathcal E^k\), conditioning
on all previous randomness.

Since
\[
\frac{1}{n}\sum_{i=1}^n\nabla f_i(y_i^{(k)})
=
\mathbb{E}_k[\bar{g}^{(k)}],
\]
\eqref{eqn:vidsgvsd} can be rewritten as
\begin{align}
f(\bar{x}^{(k+1)})
&\leq
f(\bar{y}^{(k)};Y^{(k)})
+
\theta
\left\langle
\mathbb{E}_k[\bar{g}^{(k)}],
x^\star-\bar{z}^{(k)}
\right\rangle
\nonumber\\
&\quad
+
\theta
\left\langle
\mathbb{E}_k[\bar{g}^{(k)}],
\bar{z}^{(k+1)}-x^\star
\right\rangle
+
\frac{L\theta^2}{2}
\|\bar{z}^{(k+1)}-\bar{z}^{(k)}\|^2
+
\frac{L}{2n}\|\Pi Y^{(k)}\|_F^2 .
\label{eqn:vidsadsgvsd}
\end{align}
Taking expectation in \eqref{eqn:vidsadsgvsd} with respect to the \(k\)-th
round oracle randomness and using \eqref{eqn:noise-bound}, we obtain
\begin{align}
\mathbb{E}_k[f(\bar{x}^{(k+1)})]
&\leq
f(\bar{y}^{(k)};Y^{(k)})
+
\theta
\left\langle
\mathbb{E}_k[\bar{g}^{(k)}],
x^\star-\bar{z}^{(k)}
\right\rangle
\nonumber\\
&\quad
+
\theta
\left\langle
\mathbb{E}_k[\bar{g}^{(k)}],
\mathbb{E}_k[\bar{z}^{(k+1)}]-x^\star
\right\rangle
\nonumber\\
&\quad
+
\frac{L\theta^2}{2}
\left\|
\mathbb{E}_k[\bar{z}^{(k+1)}]-\bar{z}^{(k)}
\right\|^2
+
\frac{L}{2n}\|\Pi Y^{(k)}\|_F^2
+
\frac{\theta^2\gamma^2L\sigma^2}
{2(\theta+\mu\gamma)^2nR}.
\label{eqn:vjidngfsd}
\end{align}
We next bound the terms in \eqref{eqn:vjidngfsd} separately.

For the first line in \eqref{eqn:vjidngfsd}, using
\(\theta\bar{z}^{(k)}=\bar{y}^{(k)}-(1-\theta)\bar{x}^{(k)}\), we have
\begin{align}
&f(\bar{y}^{(k)};Y^{(k)})
+
\theta
\left\langle
\mathbb{E}_k[\bar{g}^{(k)}],
x^\star-\bar{z}^{(k)}
\right\rangle
\nonumber\\
&=
f(\bar{y}^{(k)};Y^{(k)})
+
\left\langle
\mathbb{E}_k[\bar{g}^{(k)}],
\theta x^\star+(1-\theta)\bar{x}^{(k)}-\bar{y}^{(k)}
\right\rangle
\nonumber\\
&=
\theta
\left(
f(\bar{y}^{(k)};Y^{(k)})
+
\left\langle
\mathbb{E}_k[\bar{g}^{(k)}],
x^\star-\bar{y}^{(k)}
\right\rangle
\right)
\nonumber\\
&\quad
+
(1-\theta)
\left(
f(\bar{y}^{(k)};Y^{(k)})
+
\left\langle
\mathbb{E}_k[\bar{g}^{(k)}],
\bar{x}^{(k)}-\bar{y}^{(k)}
\right\rangle
\right).
\label{eqn:vndsignds}
\end{align}
By Lemma~\ref{lem:fund-smooth},
\begin{align}
\label{eqn:vmobcxcx}
f(\bar{y}^{(k)};Y^{(k)})
+
\left\langle
\mathbb{E}_k[\bar{g}^{(k)}],
x^\star-\bar{y}^{(k)}
\right\rangle
\leq
f(x^\star)
-
\frac{\mu}{2}\|\bar{y}^{(k)}-x^\star\|^2
-
\frac{\mu}{2n}\|\Pi Y^{(k)}\|_F^2 .
\end{align}
By the definition of \(D_f(\bar{x}^{(k)};Y^{(k)})\),
\begin{align}
\label{eqn:vmobcxcx2}
f(\bar{y}^{(k)};Y^{(k)})
+
\left\langle
\mathbb{E}_k[\bar{g}^{(k)}],
\bar{x}^{(k)}-\bar{y}^{(k)}
\right\rangle
=
f(\bar{x}^{(k)})-D_f(\bar{x}^{(k)};Y^{(k)}).
\end{align}
Plugging \eqref{eqn:vmobcxcx} and \eqref{eqn:vmobcxcx2} into
\eqref{eqn:vndsignds}, we get
\begin{align}
&f(\bar{y}^{(k)};Y^{(k)})
+
\theta
\left\langle
\mathbb{E}_k[\bar{g}^{(k)}],
x^\star-\bar{z}^{(k)}
\right\rangle
\nonumber\\
&\leq
\theta
\left(
f(x^\star)
-
\frac{\mu}{2}\|\bar{y}^{(k)}-x^\star\|^2
-
\frac{\mu}{2n}\|\Pi Y^{(k)}\|_F^2
\right)
\nonumber\\
&\quad
+
(1-\theta)
\left(
f(\bar{x}^{(k)})-D_f(\bar{x}^{(k)};Y^{(k)})
\right).
\label{eqn:vninfs}
\end{align}

For the second line in \eqref{eqn:vjidngfsd}, we can rewrite it as
\begin{align}
&\theta
\left\langle
\mathbb{E}_k[\bar{g}^{(k)}],
\mathbb{E}_k[\bar{z}^{(k+1)}]-x^\star
\right\rangle
\nonumber\\
&=
-\frac{\theta^2}{\gamma}
\left\langle
\mathbb{E}_k[\bar{z}^{(k+1)}]-\bar{z}^{(k)}
+
\frac{\mu\gamma}{\theta}
\left(
\mathbb{E}_k[\bar{z}^{(k+1)}]-\bar{y}^{(k)}
\right),
\mathbb{E}_k[\bar{z}^{(k+1)}]-x^\star
\right\rangle
\nonumber\\
&=
\frac{\theta^2}{2\gamma}
\left(
\|\bar{z}^{(k)}-x^\star\|^2
-
\left\|
\mathbb{E}_k[\bar{z}^{(k+1)}]-\bar{z}^{(k)}
\right\|^2
-
\left\|
\mathbb{E}_k[\bar{z}^{(k+1)}]-x^\star
\right\|^2
\right)
\nonumber\\
&\quad
+
\frac{\mu\theta}{2}
\left(
\|\bar{y}^{(k)}-x^\star\|^2
-
\left\|
\mathbb{E}_k[\bar{z}^{(k+1)}]-\bar{y}^{(k)}
\right\|^2
-
\left\|
\mathbb{E}_k[\bar{z}^{(k+1)}]-x^\star
\right\|^2
\right),
\label{eqn:vninfs2}
\end{align}
where the last identity follows from
\(-2\langle a,b\rangle=\|a-b\|^2-\|a\|^2-\|b\|^2\).

Substituting \eqref{eqn:vninfs} and \eqref{eqn:vninfs2} into
\eqref{eqn:vjidngfsd}, and dropping the nonpositive terms involving
\(\|\mathbb{E}_k[\bar{z}^{(k+1)}]-x^\star\|^2\) and
\(\|\mathbb{E}_k[\bar{z}^{(k+1)}]-\bar{y}^{(k)}\|^2\), we obtain
\begin{align}
&\mathbb{E}_k[f(\bar{x}^{(k+1)})]-f(x^\star)
+
\left(
\frac{\theta^2}{2\gamma}
+
\frac{\mu\theta}{2}
\right)
\left\|
\mathbb{E}_k[\bar{z}^{(k+1)}]-x^\star
\right\|^2
\nonumber\\
&\leq
(1-\theta)\left(f(\bar{x}^{(k)})-f(x^\star)\right)
+
\frac{\theta^2}{2\gamma}\|\bar{z}^{(k)}-x^\star\|^2
+
\frac{\theta^2L\gamma^2\sigma^2}{2(\theta+\mu\gamma)^2nR}
\nonumber\\
&\quad
-
\left(
\frac{\theta^2}{2\gamma}
-
\frac{L\theta^2}{2}
\right)
\left\|
\mathbb{E}_k[\bar{z}^{(k+1)}]-\bar{z}^{(k)}
\right\|^2
-
(1-\theta)D_f(\bar{x}^{(k)};Y^{(k)})
+
\frac{L}{2n}\|\Pi Y^{(k)}\|_F^2 .
\label{eqn:bvnidndas}
\end{align}
Taking expectation with respect to all previous randomness, and using
\[
\mathbb{E}\|\bar{z}^{(k)}-x^\star\|^2
\leq
\mathbb{E}
\left[
\left\|
\mathbb{E}_{k-1}[\bar{z}^{(k)}]-x^\star
\right\|^2
\right]
+
\frac{\gamma^2\sigma^2}{(\theta+\mu\gamma)^2nR},
\]
we obtain
\begin{align}
&\mathbb{E}[f(\bar{x}^{(k+1)})]-f(x^\star)
+
\left(
\frac{\theta^2}{2\gamma}
+
\frac{\mu\theta}{2}
\right)
\mathbb{E}
\left[
\left\|
\mathbb{E}_k[\bar{z}^{(k+1)}]-x^\star
\right\|^2
\right]
\nonumber\\
&\leq
(1-\theta)
\left(\mathbb{E}[f(\bar{x}^{(k)})]-f(x^\star)\right)
+
\frac{\theta^2}{2\gamma}
\mathbb{E}
\left[
\left\|
\mathbb{E}_{k-1}[\bar{z}^{(k)}]-x^\star
\right\|^2
\right]
\nonumber\\
&\quad
+
\frac{\theta^2\gamma\sigma^2}{2nR}
\left(
\frac{\gamma L}{(\theta+\mu\gamma)^2}
+
\frac{1}{(\theta+\mu\gamma)^2}
\right)
\nonumber\\
&\quad
-
\left(
\frac{\theta^2}{2\gamma}
-
\frac{L\theta^2}{2}
\right)
\mathbb{E}
\left[
\left\|
\mathbb{E}_k[\bar{z}^{(k+1)}]-\bar{z}^{(k)}
\right\|^2
\right]
\nonumber\\
&\quad
-
(1-\theta)
\mathbb{E}\left[D_f(\bar{x}^{(k)};Y^{(k)})\right]
+
\frac{L}{2n}
\mathbb{E}\left[\|\Pi Y^{(k)}\|_F^2\right].
\label{eqn:bvnidndasvxz}
\end{align}

Since \(\theta=\sqrt{\mu\gamma}/2\), we have \(\theta^2=\mu\gamma/4\).
Under the imposed stepsize condition, \(\theta\leq 1/2\), and hence
\[
\theta^2
\leq
\mu\gamma(1-\theta),
\]
which is equivalent to
\[
\frac{\theta^2}{\gamma}
\leq
(1-\theta)
\left(
\frac{\theta^2}{\gamma}
+
\mu\theta
\right).
\]
Therefore finish the proof of lemma~\ref{lem:descent-c-sc}.

\end{proof}

\begin{lemma}[\sc Consensus lemma]
\label{lem:consensus-c-sc}
Suppose Assumptions~\ref{ass:smooth}, \ref{ass:strongly-convex}, and
\ref{ass:stochastic-gradients} hold, and Algorithm~\ref{alg:mcadsgd} is used. Then for any $k \geq 0$,
\begin{equation}
\label{eq:psi-lemmarecursion}
\begin{aligned}
&\mathbb{E}\left[\|\Pi X^{(k+1)}\|_F^2\right]
+
\frac{4\rho^2(1+2\rho^2)}{(1-\rho^2)^2}
\theta^2
\mathbb{E}\left[\|\Pi Z^{(k+1)}\|_F^2\right]\\
&\leq
\left(
\frac{3\rho^2}{1+2\rho^2}
+
\frac{144L^2\lambda\rho^4(1+\rho^2)\gamma^2}{(1-\rho^2)^3}
\right)
\left(
\mathbb{E}\left[\|\Pi X^{(k)}\|_F^2\right]
+
\frac{4\rho^2(1+2\rho^2)}{(1-\rho^2)^2}
\theta^2
\mathbb{E}\left[\|\Pi Z^{(k)}\|_F^2\right]
\right)
\\
&\quad
+
\frac{144L\rho^4(1+\rho^2)\gamma^2}{(1-\rho^2)^3}
\mathbb{E}\left[f(\bar{x}^{(k)})-f^\star\right]
\\
&\quad
+
\frac{72nL^2\rho^4(1+\rho^2)\gamma^2}{(1-\rho^2)^3}
\mathbb{E}\left[\|\bar{x}^{(k)}-\bar{y}^{(k)}\|^2\right]
\\
&\quad
+
\frac{144nL\rho^4(1+\rho^2)\gamma^2}{(1-\rho^2)^3}\Delta_f^\star
+
\frac{24\rho^4(1+\rho^2)\gamma^2}{(1-\rho^2)^3}
\frac{(n-1)\sigma^2}{R}.
\end{aligned}
\end{equation}
\end{lemma}

\begin{proof}[Proof of Lemma~\ref{lem:consensus-c-sc}]
We start from the following consensus-error update relations:
\begin{equation}
\label{eq:update-rules}
\begin{aligned}
\|\Pi Y^{(k+1)}\|_F
&\leq
\theta\|\Pi Z^{(k+1)}\|_F
+
(1-\theta)\|\Pi X^{(k+1)}\|_F,
\\
\|\Pi Z^{(k+1)}\|_F
&\leq
\rho\left(
\frac{\mu\gamma}{\theta+\mu\gamma}\|\Pi Y^{(k)}\|_F
+
\frac{\theta}{\theta+\mu\gamma}\|\Pi Z^{(k)}\|_F
+
\frac{\gamma}{\theta+\mu\gamma}\|\Pi G^{(k)}\|_F
\right)
\\
&\leq
\rho\left(
\frac{\theta(1+\mu\gamma)}{\theta+\mu\gamma}\|\Pi Z^{(k)}\|_F
+
\frac{(1-\theta)\mu\gamma}{\theta+\mu\gamma}\|\Pi X^{(k)}\|_F
+
\frac{\gamma}{\theta+\mu\gamma}\|\Pi G^{(k)}\|_F
\right)
\\
&\leq
\rho\|\Pi Z^{(k)}\|_F
+
\frac{\rho\gamma\mu}{\theta}\|\Pi X^{(k)}\|_F
+
\frac{\rho\gamma}{\theta}\|\Pi G^{(k)}\|_F,
\\
\|\Pi X^{(k+1)}\|_F
&\leq
\rho\left(
\theta\|\Pi Z^{(k+1)}\|_F
+
\|\Pi X^{(k)}\|_F
\right),
\end{aligned}
\end{equation}
where $
\rho=\sqrt{2}\left(1-\sqrt{1-\beta}\right)^R.$
Throughout the strongly convex part of the proof, we set
$
\theta\triangleq \frac{\sqrt{\mu\gamma}}{2}.
$
The stepsize conditions imposed below ensure \(0<\theta<1/2\).

We next derive the recursion for the consensus Lyapunov term. By Young's
inequality, for any matrices \(A\) and \(B\) of the same size,
\[
(\|A\|_F+\|B\|_F)^2
\leq
\frac{2}{1+\rho^2}\|A\|_F^2
+
\frac{2}{1-\rho^2}\|B\|_F^2.
\]
Combining this inequality with \eqref{eq:update-rules}, we obtain
\begin{equation}
\label{eq:z-bound}
\begin{aligned}
\mathbb{E}\left[\|\Pi Z^{(k+1)}\|_F^2\right]
&\leq
\frac{2\rho^2}{1+\rho^2}
\mathbb{E}\left[\|\Pi Z^{(k)}\|_F^2\right]
\\
&\quad
+
\frac{4\rho^2}{1-\rho^2}
\left(
\frac{\gamma^2\mu^2}{\theta^2}
\mathbb{E}\left[\|\Pi X^{(k)}\|_F^2\right]
+
\frac{\gamma^2}{\theta^2}
\mathbb{E}\left[\|\Pi G^{(k)}\|_F^2\right]
\right).
\end{aligned}
\end{equation}
Similarly,
\begin{equation}
\label{eq:x-bound}
\begin{aligned}
\mathbb{E}\left[\|\Pi X^{(k+1)}\|_F^2\right]
&\leq
\frac{2\rho^2}{1+\rho^2}
\mathbb{E}\left[\|\Pi X^{(k)}\|_F^2\right]
+
\frac{2\rho^2\theta^2}{1-\rho^2}
\mathbb{E}\left[\|\Pi Z^{(k+1)}\|_F^2\right]
\\
&\leq
\frac{2\rho^2}{1+\rho^2}
\mathbb{E}\left[\|\Pi X^{(k)}\|_F^2\right]
+
\frac{2\rho^2\theta^2}{1-\rho^2}
\frac{2\rho^2}{1+\rho^2}
\mathbb{E}\left[\|\Pi Z^{(k)}\|_F^2\right]
\\
&\quad
+
\frac{4\rho^2}{1-\rho^2}
\frac{2\rho^2\theta^2}{1-\rho^2}
\left(
\frac{\gamma^2\mu^2}{\theta^2}
\mathbb{E}\left[\|\Pi X^{(k)}\|_F^2\right]
+
\frac{\gamma^2}{\theta^2}
\mathbb{E}\left[\|\Pi G^{(k)}\|_F^2\right]
\right)
\\
&=
\left(
\frac{2\rho^2}{1+\rho^2}
+
\frac{8\rho^4\gamma^2\mu^2}{(1-\rho^2)^2}
\right)
\mathbb{E}\left[\|\Pi X^{(k)}\|_F^2\right]
\\
&\quad
+
\frac{4\rho^4\theta^2}{1-\rho^4}
\mathbb{E}\left[\|\Pi Z^{(k)}\|_F^2\right]
+
\frac{8\rho^4\gamma^2}{(1-\rho^2)^2}
\mathbb{E}\left[\|\Pi G^{(k)}\|_F^2\right].
\end{aligned}
\end{equation}
Multiplying \eqref{eq:z-bound} by
$
\frac{4\rho^2(1+2\rho^2)}{(1-\rho^2)^2}\theta^2
$
and adding it to \eqref{eq:x-bound}, we have
\begin{equation}
\label{eq:combined-consensus-bound}
\begin{aligned}
&\mathbb{E}\left[\|\Pi X^{(k+1)}\|_F^2\right]
+
\frac{4\rho^2(1+2\rho^2)}{(1-\rho^2)^2}
\theta^2
\mathbb{E}\left[\|\Pi Z^{(k+1)}\|_F^2\right]
\\
&\leq
\left(
\frac{2\rho^2}{1+\rho^2}
+
\frac{8\rho^4\gamma^2\mu^2}{(1-\rho^2)^2}
+
\frac{16\rho^4(1+2\rho^2)\gamma^2\mu^2}{(1-\rho^2)^3}
\right)
\mathbb{E}\left[\|\Pi X^{(k)}\|_F^2\right]
\\
&\quad
+
\left(
\frac{4\rho^4(1+2\rho^2)}{(1+\rho^2)(1-\rho^2)^2}
+
\frac{4\rho^4}{1-\rho^4}
\right)
\theta^2
\mathbb{E}\left[\|\Pi Z^{(k)}\|_F^2\right]
\\
&\quad
+
\left(
\frac{8\rho^4\gamma^2}{(1-\rho^2)^2}
+
\frac{16\rho^4(1+2\rho^2)\gamma^2}{(1-\rho^2)^3}
\right)
\mathbb{E}\left[\|\Pi G^{(k)}\|_F^2\right]
\\
&=
\left(
\frac{2\rho^2}{1+\rho^2}
+
\frac{24\rho^4(1+\rho^2)\gamma^2\mu^2}{(1-\rho^2)^3}
\right)
\mathbb{E}\left[\|\Pi X^{(k)}\|_F^2\right]
\\
&\quad
+
\frac{3\rho^2}{1+2\rho^2}
\frac{4\rho^2(1+2\rho^2)}{(1-\rho^2)^2}
\theta^2
\mathbb{E}\left[\|\Pi Z^{(k)}\|_F^2\right]
\\
&\quad
+
\frac{24\rho^4(1+\rho^2)\gamma^2}{(1-\rho^2)^3}
\mathbb{E}\left[\|\Pi G^{(k)}\|_F^2\right].
\end{aligned}
\end{equation}
  
Choosing \(\gamma\) such that
$
\gamma\leq \frac{1-\rho^2}{84\rho\mu},
$
and using the additional smallness of \(\rho\) guaranteed by the choice of
\(R\), we have
\[
\frac{2\rho^2}{1+\rho^2}
+
\frac{24\rho^4(1+\rho^2)\mu^2\gamma^2}{(1-\rho^2)^3}
\leq
\frac{3\rho^2}{1+2\rho^2}.
\]
Therefore finishes the proof.
\end{proof}

With the above lemmas, we can now present the proof for Theorem~\ref{thm:paper_main_complexity}.

\begin{proof}[Proof of Theorem~\ref{thm:paper_main_complexity}]

Define
\begin{equation}
\label{eq:psi-definition}
\Psi_k
\triangleq
\mathbb{E}\left[\|\Pi X^{(k)}\|_F^2\right]
+
\frac{4\rho^2(1+2\rho^2)}{(1-\rho^2)^2}
\theta^2
\mathbb{E}\left[\|\Pi Z^{(k)}\|_F^2\right].
\end{equation}
Then lemma~\ref{lem:consensus-c-sc} implies
\begin{equation}
\label{eq:psi-recursion}
\Psi_{k+1}
\leq
\frac{3\rho^2}{1+2\rho^2}\Psi_k
+
\frac{24\rho^4(1+\rho^2)\gamma^2}{(1-\rho^2)^3}
\mathbb{E}\left[\|\Pi G^{(k)}\|_F^2\right],
\end{equation}
Moreover, by \eqref{eq:update-rules} and Young's inequality,
\begin{equation}
\label{eq:psi-y}
\mathbb{E}\left[\|\Pi Y^{(k)}\|_F^2\right]
\leq
2\left(
\mathbb{E}\left[\|\Pi X^{(k)}\|_F^2\right]
+
\theta^2\mathbb{E}\left[\|\Pi Z^{(k)}\|_F^2\right]
\right)
\leq
2\lambda\Psi_k,
\end{equation}
where
\[
\lambda
\triangleq
\max\left\{
1,\,
\frac{(1-\rho^2)^2}{4\rho^2(1+2\rho^2)}
\right\}.
\]

It remains to bound
\(\mathbb{E}\left[\|\Pi G^{(k)}\|_F^2\right]\) and then close the recursion for
\(\Psi_k\). Since
\[
\Pi=I_n-\frac{1}{n}\mathbf{1}_n\mathbf{1}_n^{\top},
\qquad
\mathbb{E}\left[
\left\|G_i^{(k)}-\nabla f_i(Y_i^{(k)})\right\|^2
\right]
=
\frac{\sigma^2}{R},
\]
we have
\begin{equation}
\label{eq:g-consensus-basic}
\begin{aligned}
\mathbb{E}\left[\|\Pi G^{(k)}\|_F^2\right]
&=
\mathbb{E}\left[\|\Pi\nabla f(Y^{(k)})\|_F^2\right]
+
\frac{(n-1)\sigma^2}{R}
\\
&=
\mathbb{E}\left[
\sum_{i=1}^n\|\nabla f_i(y_i^{(k)})\|^2
-
\frac{1}{n}
\left\|\sum_{i=1}^n\nabla f_i(y_i^{(k)})\right\|^2
\right]
+
\frac{(n-1)\sigma^2}{R}
\\
&\leq
\mathbb{E}\left[
\sum_{i=1}^n\|\nabla f_i(y_i^{(k)})\|^2
\right]
+
\frac{(n-1)\sigma^2}{R}.
\end{aligned}
\end{equation}

We first record the following auxiliary inequality: for any \(y\in\mathbb{R}^d\),
\begin{equation}
\label{eq:local-gradient-growth}
\sum_{i=1}^n\|\nabla f_i(y)\|^2
\leq
2nL\bigl(f(y)-f^\star\bigr)
+
2nL\Delta_f^\star,
\end{equation}
where
\[
f^\star\triangleq \min_{x\in\mathbb{R}^d} f(x),
\qquad
f_i^\star\triangleq \min_{x\in\mathbb{R}^d} f_i(x),
\qquad
\Delta_f^\star
\triangleq
f^\star-\frac{1}{n}\sum_{i=1}^n f_i^\star .
\]
Indeed, by \(L\)-smoothness, for all \(i\in[n]\) and any
\(y,z\in\mathbb{R}^d\),
\[
f_i(z)
\leq
f_i(y)
+
\langle \nabla f_i(y),z-y\rangle
+
\frac{L}{2}\|z-y\|^2.
\]
Taking \(z=y-\frac{1}{L}\nabla f_i(y)\) and using
\(f_i(z)\geq f_i^\star\), we obtain
\[
f_i^\star
\leq
f_i(y)
-
\frac{1}{2L}\|\nabla f_i(y)\|^2.
\]
Rearranging and summing over \(i=1,\ldots,n\) gives
\eqref{eq:local-gradient-growth}.

Now, using \(L\)-smoothness and \eqref{eq:local-gradient-growth}, we have
\begin{equation}
\label{eq:local-gradient-at-y}
\begin{aligned}
\sum_{i=1}^n \|\nabla f_i(y_i^{(k)})\|^2
&=
\sum_{i=1}^n
\left\|
\nabla f_i(y_i^{(k)})
-
\nabla f_i(\bar{y}^{(k)})
+
\nabla f_i(\bar{y}^{(k)})
-
\nabla f_i(\bar{x}^{(k)})
+
\nabla f_i(\bar{x}^{(k)})
\right\|^2
\\
&\leq
3\sum_{i=1}^n
\Bigl[
\|\nabla f_i(y_i^{(k)})-\nabla f_i(\bar{y}^{(k)})\|^2
+
\|\nabla f_i(\bar{y}^{(k)})-\nabla f_i(\bar{x}^{(k)})\|^2
\\
&\hspace{6.8cm}
+
\|\nabla f_i(\bar{x}^{(k)})\|^2
\Bigr]
\\
&\leq
3L^2\|\Pi Y^{(k)}\|_F^2
+
3nL^2\|\bar{x}^{(k)}-\bar{y}^{(k)}\|^2
+
6nL\bigl(f(\bar{x}^{(k)})-f^\star\bigr)
+
6nL\Delta_f^\star .
\end{aligned}
\end{equation}
Combining \eqref{eq:g-consensus-basic} and \eqref{eq:local-gradient-at-y}
yields
\begin{equation}
\label{eq:g-bound}
\begin{aligned}
\mathbb{E}\left[\|\Pi G^{(k)}\|_F^2\right]
&\leq
3L^2\mathbb{E}\left[\|\Pi Y^{(k)}\|_F^2\right]
+
3nL^2\mathbb{E}\left[\|\bar{x}^{(k)}-\bar{y}^{(k)}\|^2\right]
\\
&\quad
+
6nL\,\mathbb{E}\left[f(\bar{x}^{(k)})-f^\star\right]
+
6nL\Delta_f^\star
+
\frac{(n-1)\sigma^2}{R}.
\end{aligned}
\end{equation}
Substituting \eqref{eq:g-bound} into \eqref{eq:psi-recursion} and then using
\eqref{eq:psi-y}, we obtain
\begin{equation}
\label{eq:psi-recursion-expanded}
\begin{aligned}
\Psi_{k+1}
&\leq
\frac{3\rho^2}{1+2\rho^2}\Psi_k
+
\frac{24\rho^4(1+\rho^2)\gamma^2}{(1-\rho^2)^3}
\Bigl(
3L^2\mathbb{E}\left[\|\Pi Y^{(k)}\|_F^2\right]
\\
&\hspace{3.8cm}
+
3nL^2\mathbb{E}\left[\|\bar{x}^{(k)}-\bar{y}^{(k)}\|^2\right]
+
6nL\,\mathbb{E}\left[f(\bar{x}^{(k)})-f^\star\right]
\Bigr)
\\
&\quad
+
\frac{24\rho^4(1+\rho^2)\gamma^2}{(1-\rho^2)^3}
\left(
6nL\Delta_f^\star
+
\frac{(n-1)\sigma^2}{R}
\right)
\\
&\leq
\left(
\frac{3\rho^2}{1+2\rho^2}
+
\frac{144L^2\lambda\rho^4(1+\rho^2)\gamma^2}{(1-\rho^2)^3}
\right)\Psi_k
+
\frac{144L\rho^4(1+\rho^2)\gamma^2}{(1-\rho^2)^3}
\mathbb{E}\left[f(\bar{x}^{(k)})-f^\star\right]
\\
&\quad
+
\frac{72nL^2\rho^4(1+\rho^2)\gamma^2}{(1-\rho^2)^3}
\mathbb{E}\left[\|\bar{x}^{(k)}-\bar{y}^{(k)}\|^2\right]
\\
&\quad
+
\frac{144nL\rho^4(1+\rho^2)\gamma^2}{(1-\rho^2)^3}\Delta_f^\star
+
\frac{24\rho^4(1+\rho^2)\gamma^2}{(1-\rho^2)^3}
\frac{(n-1)\sigma^2}{R}.
\end{aligned}
\end{equation}

By lemma~\ref{lem:descent-c-sc} in the strongly convex case, we have
\begin{equation}
\label{eq:descent-raw}
\begin{aligned}
&\mathbb{E}\left[f(\bar{x}^{(k+1)})\right]-f(x^\star)
+
\left(
\frac{\theta^2}{2\gamma}
+
\frac{\mu\theta}{2}
\right)
\mathbb{E}\left[
\left\|
\mathbb{E}_k[\bar{z}^{(k+1)}]-x^\star
\right\|^2
\right]
\\
&\leq
(1-\theta)
\left(
\mathbb{E}\left[f(\bar{x}^{(k)})\right]-f(x^\star)
\right)
+
\frac{\theta^2}{2\gamma}
\mathbb{E}\left[
\left\|
\mathbb{E}_{k-1}[\bar{z}^{(k)}]-x^\star
\right\|^2
\right]
\\
&\quad
+
\frac{\theta^2\gamma\sigma^2}{2nR}
\left(
\frac{\gamma L}{(\theta+\mu\gamma)^2}
+
\frac{1}{(\theta+\mu\gamma)^2}
\right)
-
\left(
\frac{\theta^2}{2\gamma}
-
\frac{L\theta^2}{2}
\right)
\mathbb{E}\left[
\left\|
\mathbb{E}_k[\bar{z}^{(k+1)}]-\bar{z}^{(k)}
\right\|^2
\right]
\\
&\quad
-
(1-\theta)\mathbb{E}\left[D_f(\bar{x}^{(k)};Y^{(k)})\right]
+
\frac{L}{2n}
\mathbb{E}\left[\|\Pi Y^{(k)}\|_F^2\right].
\end{aligned}
\end{equation}
Furthermore, by the definition of $D_f$, 
\begin{equation}
\label{eq:Df-lower-bound}
D_f(\bar{x}^{(k)};Y^{(k)})
\geq
\frac{\mu}{2}
\|\bar{x}^{(k)}-\bar{y}^{(k)}\|^2
+
\frac{\mu}{2n}
\|\Pi Y^{(k)}\|_F^2 .
\end{equation}
Substituting \eqref{eq:Df-lower-bound} into \eqref{eq:descent-raw} gives
\begin{equation}
\label{eq:descent-with-consensus}
\begin{aligned}
&\mathbb{E}\left[f(\bar{x}^{(k+1)})\right]-f(x^\star)
+
\left(
\frac{\theta^2}{2\gamma}
+
\frac{\mu\theta}{2}
\right)
\mathbb{E}\left[
\left\|
\mathbb{E}_k[\bar{z}^{(k+1)}]-x^\star
\right\|^2
\right]
\\
&\leq
(1-\theta)
\left(
\mathbb{E}\left[f(\bar{x}^{(k)})\right]-f(x^\star)
\right)
+
(1-\theta)
\left(
\frac{\theta^2}{2\gamma}
+
\frac{\mu\theta}{2}
\right)
\mathbb{E}\left[
\left\|
\mathbb{E}_{k-1}[\bar{z}^{(k)}]-x^\star
\right\|^2
\right]
\\
&\quad
+
\left(
\frac{L}{2n}
-
(1-\theta)\frac{\mu}{2n}
\right)
\mathbb{E}\left[\|\Pi Y^{(k)}\|_F^2\right]
-
(1-\theta)\frac{\mu}{2}
\mathbb{E}\left[\|\bar{x}^{(k)}-\bar{y}^{(k)}\|^2\right]
\\
&\quad
+
\frac{\theta^2\gamma\sigma^2}{2nR}
\left(
\frac{\gamma L}{(\theta+\mu\gamma)^2}
+
\frac{1}{(\theta+\mu\gamma)^2}
\right)
\\
&\leq
(1-\theta)
\left(
\mathbb{E}\left[f(\bar{x}^{(k)})\right]-f(x^\star)
\right)
+
(1-\theta)
\left(
\frac{\theta^2}{2\gamma}
+
\frac{\mu\theta}{2}
\right)
\mathbb{E}\left[
\left\|
\mathbb{E}_{k-1}[\bar{z}^{(k)}]-x^\star
\right\|^2
\right]
\\
&\quad
+
2\lambda
\left(
\frac{L}{2n}
-
(1-\theta)\frac{\mu}{2n}
\right)\Psi_k
-
(1-\theta)\frac{\mu}{2}
\mathbb{E}\left[\|\bar{x}^{(k)}-\bar{y}^{(k)}\|^2\right]
\\
&\quad
+
\frac{\theta^2\gamma\sigma^2}{2nR}
\left(
\frac{\gamma L}{(\theta+\mu\gamma)^2}
+
\frac{1}{(\theta+\mu\gamma)^2}
\right).
\end{aligned}
\end{equation}

To simplify notation, define
\[
\alpha(\rho,\gamma)
\triangleq
\frac{3\rho^2}{1+2\rho^2}
+
\frac{144L^2\lambda\rho^4(1+\rho^2)\gamma^2}{(1-\rho^2)^3},
\]
\[
B(\rho,\gamma)
\triangleq
\frac{144L\rho^4(1+\rho^2)\gamma^2}{(1-\rho^2)^3},
\qquad
\tau(\rho,\gamma)
\triangleq
\frac{72nL^2\rho^4(1+\rho^2)\gamma^2}{(1-\rho^2)^3},
\]
\[
C(\rho,\gamma)
\triangleq
\frac{144nL\rho^4(1+\rho^2)\gamma^2}{(1-\rho^2)^3}\Delta_f^\star
+
\frac{24\rho^4(1+\rho^2)\gamma^2}{(1-\rho^2)^3}
\frac{(n-1)\sigma^2}{R},
\]
\[
D(\rho,\gamma)
\triangleq
\frac{\theta^2\gamma\sigma^2}{2nR}
\left(
\frac{\gamma L}{(\theta+\mu\gamma)^2}
+
\frac{1}{(\theta+\mu\gamma)^2}
\right),
\]
and
\[
F_k
\triangleq
\mathbb{E}\left[f(\bar{x}^{(k)})\right]-f(x^\star)
+
\left(
\frac{\theta^2}{2\gamma}
+
\frac{\mu\theta}{2}
\right)
\mathbb{E}\left[
\left\|
\mathbb{E}_{k-1}[\bar{z}^{(k)}]-x^\star
\right\|^2
\right].
\]

Then the previous estimates in lemma~\ref{lem:descent-c-sc} and  lemma~\ref{lem:consensus-c-sc} can be summarized as
\begin{equation}
\label{eq:coupled-recursion}
\left\{
\begin{aligned}
\Psi_{k+1}
&\leq
\alpha(\rho,\gamma)\Psi_k
+
B(\rho,\gamma)F_k
+
C(\rho,\gamma)
+
\tau(\rho,\gamma)
\mathbb{E}\left[\|\bar{x}^{(k)}-\bar{y}^{(k)}\|^2\right],
\\
F_{k+1}
&\leq
(1-\theta)F_k
+
\frac{\lambda L}{n}\Psi_k
+
D(\rho,\gamma)
-
(1-\theta)\frac{\mu}{2}
\mathbb{E}\left[\|\bar{x}^{(k)}-\bar{y}^{(k)}\|^2\right].
\end{aligned}
\right.
\end{equation}
Therefore, for any positive constant \(u>0\),
\begin{equation}
\label{eq:weighted-recursion}
\begin{aligned}
F_{k+1}+u\Psi_{k+1}
&\leq
(uB+1-\theta)F_k
+
\left(\frac{\lambda L}{n}+\alpha u\right)\Psi_k
+
uC+D
\\
&\quad
-
\left(
(1-\theta)\frac{\mu}{2}
-
u\tau
\right)
\mathbb{E}\left[\|\bar{x}^{(k)}-\bar{y}^{(k)}\|^2\right].
\end{aligned}
\end{equation}
If \(u\) is chosen such that
\begin{equation}
\label{eq:u-conditions}
\left\{
\begin{aligned}
\frac{\lambda L}{n}+\alpha u
&\leq
u(uB+1-\theta),
\\
uB
&\leq
\frac{\theta}{2},
\\
u\tau
&\leq
(1-\theta)\frac{\mu}{2},
\end{aligned}
\right.
\end{equation}
then
\begin{equation}
\label{eq:contracted-recursion}
F_{k+1}+u\Psi_{k+1}
\leq
\left(1-\frac{\theta}{2}\right)
(F_k+u\Psi_k)
+
uC+D .
\end{equation}
Iterating \eqref{eq:contracted-recursion} gives
\begin{equation}
\label{eq:iterated-recursion}
F_{k+1}+u\Psi_{k+1}
\leq
\left(1-\frac{\theta}{2}\right)^{k+1}
(F_0+u\Psi_0)
+
\sum_{t=0}^k
\left(1-\frac{\theta}{2}\right)^t
(uC+D).
\end{equation}

We now bound the initial term and the accumulated error term. Since
\(\Psi_0=0\), \(\theta^2/(2\gamma)=\mu/8\), and
\(\mu\theta/2\leq \mu/4\) under \(\theta\leq 1/2\),
\begin{equation}
\label{eq:F0-bound}
\begin{aligned}
F_0
&=
f(\bar{x}^{(0)})-f(x^\star)
+
\left(
\frac{\theta^2}{2\gamma}
+
\frac{\mu\theta}{2}
\right)
\|\bar{z}^{(0)}-x^\star\|^2
\\
&\leq
\frac{L}{2}\Delta_x
+
\left(
\frac{\mu}{8}
+
\frac{\mu}{4}
\right)\Delta_x
\leq
L\Delta_x .
\end{aligned}
\end{equation}
Moreover,
\begin{equation}
\label{eq:uC-bound}
uC
=
\frac{L}{(1-\theta)(1+2\rho^2)}
\left(
\frac{72L\rho^2(1+\rho^2)\gamma^2}{1-\rho^2}\Delta_f^\star
+
\frac{12\rho^2(1+\rho^2)\gamma^2}{1-\rho^2}
\frac{(n-1)\sigma^2}{nR}
\right),
\end{equation}
and
\begin{equation}
\label{eq:D-bound}
\begin{aligned}
D
&=
\frac{\theta^2\gamma\sigma^2}{2nR}
\frac{1+\gamma L}{(\theta+\mu\gamma)^2}\\
&=
\frac{\theta^2\gamma\sigma^2}{2nR}
\frac{1+\gamma L}{\mu\gamma(1/2+\sqrt{\mu\gamma})^2}
\\
&\leq
\frac{8\theta^2\sigma^2}{3nR\mu},
\end{aligned}
\end{equation}
where the last inequality follows from \(\gamma<1/(3L)\). Therefore,
\begin{equation}
\label{eq:sum-error-bound}
\begin{aligned}
&\sum_{t=0}^k
\left(1-\frac{\theta}{2}\right)^t
(uC+D)\\
&\leq
\frac{2}{\theta}
\left[
\frac{L}{(1-\theta)(1+2\rho^2)}
\left(
\frac{72L\rho^2(1+\rho^2)\gamma^2}{1-\rho^2}\Delta_f^\star
+
\frac{12\rho^2(1+\rho^2)\gamma^2}{1-\rho^2}
\frac{(n-1)\sigma^2}{nR}
\right)
+
\frac{8\theta^2\sigma^2}{3nR\mu}
\right]
\\
&=
\frac{2}{\theta}
\frac{L}{(1-\theta)(1+2\rho^2)}
\left(
\frac{72L\rho^2(1+\rho^2)\gamma^2}{1-\rho^2}\Delta_f^\star
+
\frac{12\rho^2(1+\rho^2)\gamma^2}{1-\rho^2}
\frac{(n-1)\sigma^2}{nR}
\right)
+
\frac{16\theta\sigma^2}{3nR\mu}.
\end{aligned}
\end{equation}

Choose
\begin{equation}
\label{eq:gamma-choice}
\gamma
=
\min\left\{
\frac{1}{3L},
\,
\frac{16R^2}{T^2\mu}
\left[
\log\left(
\frac{nT\mu L\Delta_x}{\sigma^2}
\right)
\right]^2
\right\},
\end{equation}
and
\begin{equation}
\label{eq:R-choice}
R
=
\left\lceil
\max\left\{
\frac{-2\log\sqrt{1-\beta}}{\sqrt{1-\beta}},
\,
\frac{-\log\frac{\sigma^2}{nL\Delta_f^\star}}{\sqrt{1-\beta}},
\,
\frac{-\log\frac{5}{1536}}{2\sqrt{1-\beta}},
\,
\frac{\log(2\sqrt{2}n)}{\sqrt{1-\beta}},
\,
\frac{-\log\left(\frac{15n}{18432}\sqrt{\frac{\mu}{L}}\right)}{\sqrt{1-\beta}},
\,
\frac{-\log\frac{15\mu}{9214L}}{\sqrt{1-\beta}}
\right\}
\right\rceil .
\end{equation}
Since
\[
\rho
=
\sqrt{2}\left(1-\sqrt{1-\beta}\right)^R
\leq
\sqrt{2}e^{-\sqrt{1-\beta}R},
\]
and using \(\theta<1/2\), \(\rho\leq 1/(2n)\), and
\(\gamma<1/(3L)\), we have
\begin{equation}
\label{eq:Gstar-term-bound}
\begin{aligned}
\frac{2}{\theta}
\frac{L}{(1-\theta)(1+2\rho^2)}
\frac{72L\rho^2(1+\rho^2)\gamma^2}{1-\rho^2}
\Delta_f^\star
&\leq
1440L^2\rho^2\gamma \Delta_f^\star
\sqrt{\frac{\gamma}{\mu}}
\\
&\leq
480L\rho^2\Delta_f^\star
\sqrt{\frac{\gamma}{\mu}} .
\end{aligned}
\end{equation}
Since
\[
R
\geq
\left\lceil
\max\left\{
\frac{-2\log\sqrt{1-\beta}}{\sqrt{1-\beta}},
\,
\frac{-\log\frac{\sigma^2}{nL\Delta_f^\star}}{\sqrt{1-\beta}}
\right\}
\right\rceil,
\]
we have
\[
\rho R
=
\sqrt{2}e^{-\sqrt{1-\beta}R}R
\leq 2,
\qquad
\rho
\leq
\frac{\sigma^2}{nL\Delta_f^\star}.
\]
Therefore,
\[
L\rho^2\Delta_f^\star
\leq
\frac{2\sigma^2}{nR},
\]
and hence
\begin{equation}
\label{eq:Gstar-term-final}
480L\rho^2\Delta_f^\star
\sqrt{\frac{\gamma}{\mu}}
\leq
\frac{720\sigma^2}{nR}
\sqrt{\frac{\gamma}{\mu}} .
\end{equation}
Similarly,
\begin{equation}
\label{eq:variance-term-bound}
\begin{aligned}
\frac{2}{\theta}
\frac{L}{(1-\theta)(1+2\rho^2)}
\frac{12\rho^2(1+\rho^2)\gamma^2}{1-\rho^2}
\frac{(n-1)\sigma^2}{nR}
&\leq
\frac{240L\rho^2\gamma(n-1)\sigma^2}{nR}
\sqrt{\frac{\gamma}{\mu}}
\\
&\leq
\frac{20\sigma^2}{nR}
\sqrt{\frac{\gamma}{\mu}} .
\end{aligned}
\end{equation}
Combining the preceding estimates yields
\begin{equation}
\label{eq:finite-time-bound}
\mathbb{E}\left[f(\bar{x}^{(k+1)})-f^\star\right]
\leq
\left(1-\frac{\sqrt{\mu\gamma}}{4}\right)^{k+1}
L\Delta_x
+
\frac{2228\sigma^2}{3nR}
\sqrt{\frac{\gamma}{\mu}},
\end{equation}
where \(k=T/R\). Finally, we obtain
\begin{equation}
\label{eq:final-rate}
\mathbb{E}\left[f(\bar{x}^{(k+1)})-f^\star\right]
=
\mathcal{O}\left(
\left(
1-q\sqrt{\frac{\mu}{L}}
\right)^{T/R}
+
\frac{\sigma^2}{n\mu T}
\log\frac{nT\mu L\Delta_x}{\sigma^2}
\right),
\end{equation}
where \(q>0\) is an absolute constant. Therefore, to guarantee
\[
\mathbb{E}\left[f(\bar{x}^{(k)})-f^\star\right]
=
\mathcal{O}(\epsilon),
\]
it suffices to require
\[
\log\left(\frac{n\mu L\Delta_x}{\sigma^2}\right)
\frac{\sigma^2}{\mu nT}
=
\mathcal{O}(\epsilon),
\qquad
\frac{\log T}{\mu nT}
=
\mathcal{O}(\epsilon),
\qquad
\left(
1-q\sqrt{\frac{\mu}{L}}
\right)^{T/R}
L\Delta_x
=
\mathcal{O}(\epsilon).
\]
These conditions are satisfied by
\begin{equation}
\label{eq:T-complexity}
\begin{aligned}
T
&=
\mathcal{O}\left(
\frac{
\log\left(\frac{n\mu L\Delta_x}{\sigma^2}\right)\sigma^2
}{
\mu n\epsilon
}
\right)
+
\mathcal{O}\left(
\frac{
\log\left(\frac{n\mu}{\epsilon}\right)
}{
\mu n\epsilon
}
\right)
+
\mathcal{O}\left(
R\sqrt{\frac{L}{\mu}}
\log\left(\frac{L\Delta_x}{\epsilon}\right)
\right)
\\
&=
\widetilde{\mathcal{O}}\left(
\frac{
\sigma^2
}{
\mu n\epsilon
}\log\left(\frac{L\Delta_x}{\epsilon}\right)
+
\sqrt{\frac{L/\mu}{1-\beta}}
\log\left(\frac{L\Delta_x}{\epsilon}\right)
\right)
\end{aligned}
\end{equation}
where the passage to the \(\widetilde{\mathcal O}\)-bound follows from
\[
R
=
\widetilde{\mathcal{O}}\left(
\frac{1}{\sqrt{1-\beta}}
\right).
\]
\end{proof}

\bibliographystyle{plain}
\bibliography{references}
\end{document}